\documentclass[conference]{ieeeconf}

\IEEEoverridecommandlockouts
\usepackage[utf8]{inputenc}
\usepackage{comment}
\usepackage{graphicx}
\usepackage{xcolor}
\usepackage{wrapfig}
\usepackage{lipsum}
\definecolor{rred}{RGB}{169, 50, 38}
\definecolor{ggreen}{RGB}{34, 153, 84}
\definecolor{bblue}{RGB}{36, 113, 163}
\definecolor{ppurple}{RGB}{125, 60, 152}
\definecolor{yyellow}{RGB}{214, 137, 16}
\definecolor{ggrey}{RGB}{112, 123, 124}

\usepackage{todonotes}

\usepackage{mathtools}
\usepackage{amsmath}
\usepackage{amssymb}
\usepackage{amsthm}
\theoremstyle{definition}
\usepackage{tikz}
\usetikzlibrary{automata,positioning,angles,quotes}
\usepackage{tikz-cd}
\usepackage[cal=cm, scr=dutchcal]{mathalfa}

\newtheoremstyle{main}
{0.3em} 
{0.3em} 
{\normalfont} 
{0pt} 
{\bfseries} 
{} 
{1pt} 
{\thmname{#1}\thmnumber{ #2. }\thmnote{\normalfont\itshape (#3)\normalfont:\\}}
\newtheoremstyle{math}
{0.2em} 
{0.2em} 
{\itshape} 
{0pt} 
{\bfseries} 
{} 
{1pt} 
{\thmname{#1}\thmnumber{ #2. }\thmnote{\normalfont\itshape (#3)\normalfont:\\}}
\theoremstyle{main}
\newtheorem{definition}{Definition} 

\theoremstyle{math}
\newtheorem{theorem}{Theorem} 

\makeatletter
\let\NAT@parse\undefined
\makeatother
\usepackage{cite}
\usepackage[pdfa,colorlinks,bookmarks=true,bookmarksopen,bookmarksnumbered,allcolors=ggreen]{hyperref}

\usepackage[algoruled,vlined,linesnumbered]{algorithm2e}

\SetAlgoSkip{SkipBeforeAndAfter}

\SetCommentSty{mycommfont}
\SetKw{Continue}{continue}
\SetAlgorithmName{Algorithm}{Alg.}

\usepackage{booktabs}
\usepackage{multirow}
\usepackage{multicol}
\usepackage{makecell}
\usepackage{nicematrix}
\usepackage{tikz}

\usepackage[font=footnotesize]{caption}
\usepackage[font=footnotesize]{subcaption}
\usepackage[export]{adjustbox}

\usepackage[
activate   = {true},
protrusion = true,
expansion  = true,
kerning    = true,
spacing    = true,
tracking   = false,
auto       = true,
selected   = true,
factor     = 1000,
stretch    = 25,
shrink     = 25,
]{microtype}

\usepackage{xspace}

\newcommand{\method}{\textsc{CoAd}\xspace}

\newcommand{\map}{\mathcal{M}\xspace}
\newcommand{\adapt}{\mathscr{A}\xspace}
\newcommand{\library}{\ensuremath{L}\xspace}
\newcommand{\rt}{\text{root}}
\newcommand{\trajrep}{\ensuremath{\pi}\xspace}
\newcommand{\T}{\mathbf{T}\xspace}

\newcommand{\start}{\text{start}}
\newcommand{\goal}{\text{goal}}

\newcommand{\DOF}{\textsc{DOF}\xspace}

\newcommand{\R}{{\ensuremath{\mathbb{R}}}\xspace}
\newcommand{\SEThree}{\ensuremath{SE(3)}\xspace}

\newcommand{\C}{\ensuremath{\mathcal{C}}\xspace}
\newcommand{\Cfree}{\ensuremath{\C_{\text{free}}}\xspace}
\newcommand{\Cobs}{\ensuremath{\C_{\text{obs}}}\xspace}
\newcommand{\Cgoal}{\ensuremath{\C_{\text{goal}}}\xspace}

\newcommand{\qs}{\ensuremath{q_\start}}

\newcommand{\qtraj}{\ensuremath{\pi}}

\newcommand{\Ts}{{\ensuremath{\mathcal{T}}}\xspace}

\newcommand{\fk}{\mathtt{FK}}
\newcommand{\ik}{\mathtt{IK}}

\title{\LARGE \bf
\method: Constant-Time Planning for Continuous Goal Manipulation \\ 
with Compressed Library and Online Adaptation
}

\author{
  Adil Shiyas$^*$, Zhuoyun Zhong$^*$, and Constantinos Chamzas
  \thanks{%
   All authors are affiliated with the Department of Robotics Engineering, Worcester Polytechnic Institute (WPI), Worcester, MA 01609, USA {\tt\small \{mshiyas, zzhong3, cchamzas\} @ wpi.edu}.
  }
  \thanks{$^{*}$ Adil Shiyas and Zhuoyun Zhong have equal contributions.}
}

\begin{document}

\maketitle
\thispagestyle{empty}
\pagestyle{empty}

\begin{abstract}
%
In many robotic manipulation tasks, the robot repeatedly solves motion-planning problems that differ mainly in the location of the goal object and its associated obstacle, while the surrounding workspace remains fixed.
Prior works have shown that leveraging experience and offline computation can accelerate repeated planning queries, but they lack guarantees of covering the continuous task space and require storing large libraries of solutions.
In this work, we present \method, a framework that provides constant-time planning over a continuous goal-parameterized task space.
\method discretizes the continuous task space into finitely many Task Coverage Regions.
Instead of planning and storing solutions for every region offline, it constructs a compressed library by only solving representative root problems.
Other problems are handled through fast adaptation from these root solutions.
At query time, the system retrieves a root motion in constant time and adapts it to the desired goal using lightweight adaptation modules such as linear interpolation, Dynamic Movement Primitives, or simple trajectory optimization.
We evaluate the framework on various manipulators and environments in simulation and the real world, showing that \method achieves substantial compression of the motion library while maintaining high success rates and sub-millisecond-level queries, outperforming baseline methods in both efficiency and path quality.
The source code is available at  
\url{https://github.com/elpis-lab/CoAd}.
\end{abstract}

\section{Introduction}

Motion planning is a central problem in robotics, enabling autonomous systems to plan collision-free motions while satisfying kinematic and geometric constraints \cite{choset2005principles}. In robotic manipulation, this problem becomes particularly challenging due to high-dimensional configuration spaces, complex robot geometry, and cluttered workspaces. 
Despite these challenges, significant progress has been made over the years~\cite{bekris2024state}, and recent efforts have focused on leveraging experience and pre-computation to accelerate repeated queries~\cite{ishida2025coverlib}.

Library-based and experience-based \cite{experience_framework, ishida2025coverlib, experience_sparse_roadmap, ERT} planners store and adapt previously computed solutions, achieving substantial speedups.
However, despite their practical effectiveness, these methods typically provide no guarantees that all task instances drawn from a continuous distribution will be solvable. 

\begin{figure}[ht!]
    \centering
    \includegraphics[width=0.78\linewidth]{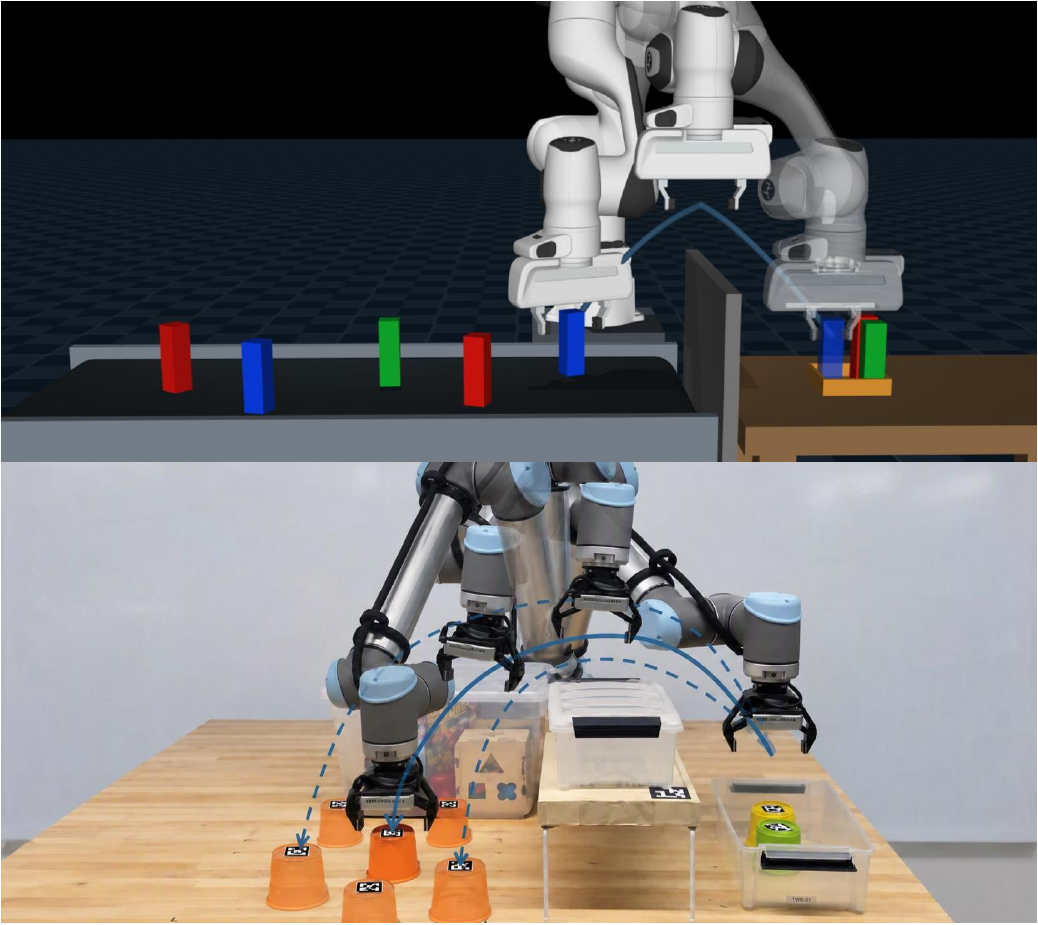}
    \caption{Robots repeatably solve motion-planning problems to grasp an object whose pose continuously varies across queries, while the rest of the workspace remains unchanged.} 
    \label{fig:intro}
\vspace{-5ex}
\end{figure}

In this work, we consider a structured yet practically important setting in which the distribution of planning problems is known.
Specifically, we study manipulation tasks such as the tabletop picking scenario shown in \autoref{fig:intro}, where a robot must repeatedly grasp a known object placed at varying poses while other workspace obstacles remain fixed.
Such settings commonly arise in production workcells for tasks such as kitting, packaging, and pick-and-place.
We refer to this family of problems as \textbf{goal-varying motion planning}, where each planning query differs only in the pose of a known object within a bounded region of $\SEThree$.

This setting induces an infinite family of motion planning problems due to the \emph{continuous goal} pose.
\method addresses this by leveraging the following key property: if a path is valid for a nominal object pose and remains valid under bounded perturbations of that pose, then that single path certifies feasibility for an entire continuous subset of nearby object placements.
This property allows the continuous goal space to be covered by finitely many regions, where a single path certifies feasibility for all goals within each region.
Furthermore, rather than computing a distinct motion plan for every region, \method constructs a \emph{compressed library} of root paths. 
The remaining regions are handled through fast \emph{adaptation} of these root paths toward their region goals, while preserving feasibility within their certified coverage.
At query time, \method identifies the region containing the query goal, retrieves the corresponding root path, and adapts it online.

To the best of our knowledge, this is the first framework that provides constant-time guarantees for continuous goal-parameterized motion planning domains. Prior constant-time approaches assume a discrete representation of the goal space~\cite{constant_belt}. Concretely, our contributions  are:
\begin{itemize}
    \item 
    A partitioning framework that enumerates an infinite family of planning queries to a finite set of paths, each associated with a certified region of goal poses.

    \item We demonstrate how an adapt and reuse framework can compress the number of paths while preserving coverage of the continuous task domain.
    
    \item We demonstrate the effectiveness of \method on manipulation tasks with 7-DOF and 8-DOF manipulators, showing compression of up to 95\% of the full path library while achieving query times as low as 0.05 ms.
\end{itemize}
\section{Related Work}

\subsection{Classical Motion Planning}

Motion planning has been studied extensively for several decades. Existing planners broadly fall into three categories. Sampling-based motion planners (SBMPs), such as PRM \cite{kavraki1998analysis} and RRT-Connect \cite{rrt-connect}, are \emph{probabilistically complete}: if a solution exists, they will find it with probability one as the computation time approaches infinity. 
However, they provide no deterministic runtime guarantees. Graph-based planners are \emph{resolution complete}, guaranteeing a solution at a fixed discretization resolution, but they scale poorly in high-dimensional manipulation spaces. 
Optimization-based planners \cite{trajopt, chomp} often achieve fast performance in practice but lack global guarantees and may converge to local minima. While these methods are general-purpose and applicable to arbitrary motion planning problems, they do not provide guarantees over structured families of problems drawn from a continuous task distribution.

\subsection{Experience-Based Motion Planning}

Experience-based motion planning seeks to accelerate repeated queries by equipping robots with memory and reusing prior solutions in related environments.
Rather than planning from scratch, these methods retrieve and adapt stored trajectories \cite{experience_framework,ishida2025coverlib,memory_of_motion}, roadmaps or learned sampling strategies \cite{chamzas2021-learn-sampling}.

These approaches have demonstrated substantial speedups in challenging manipulation tasks, and our method adopts a similar retrieve-and-adapt strategy to CoverLib~\cite{ishida2025coverlib}.
However, while these methods operate over continuous problem domains, they do not provide deterministic guarantees that the entire continuous task domain is covered.

\subsection{Constant-Time and Precomputation-Based Planning}

More closely related to our work are fixed-time motion planning approaches~\cite{constant_belt,islam2021alternative,cover,van2005roadmap}, which shift computational effort to an offline phase to guarantee bounded query-time performance. These methods typically pre-compute solutions over a finite or pre-discretized set of task or obstacle configurations.
In particular~\cite{constant_belt} iteratively select uncovered discrete goals, compute root paths, and reuse them to cover nearby goals. Our library construction adopts this general coverage-based strategy. The key distinction is that we first derive a finite, certifiably covered representation from an underlying continuous goal space.
Our framework therefore provides constant-time guarantees over a continuous goal-parameterized domain. The work in~\cite{cover} also operates over a continuous domain, but considers varying obstacles with a fixed goal, whereas we consider a continuously varying goal and its associated obstacle.
\begin{figure*}[ht]
    \centering
    \includegraphics[width=0.9\linewidth]{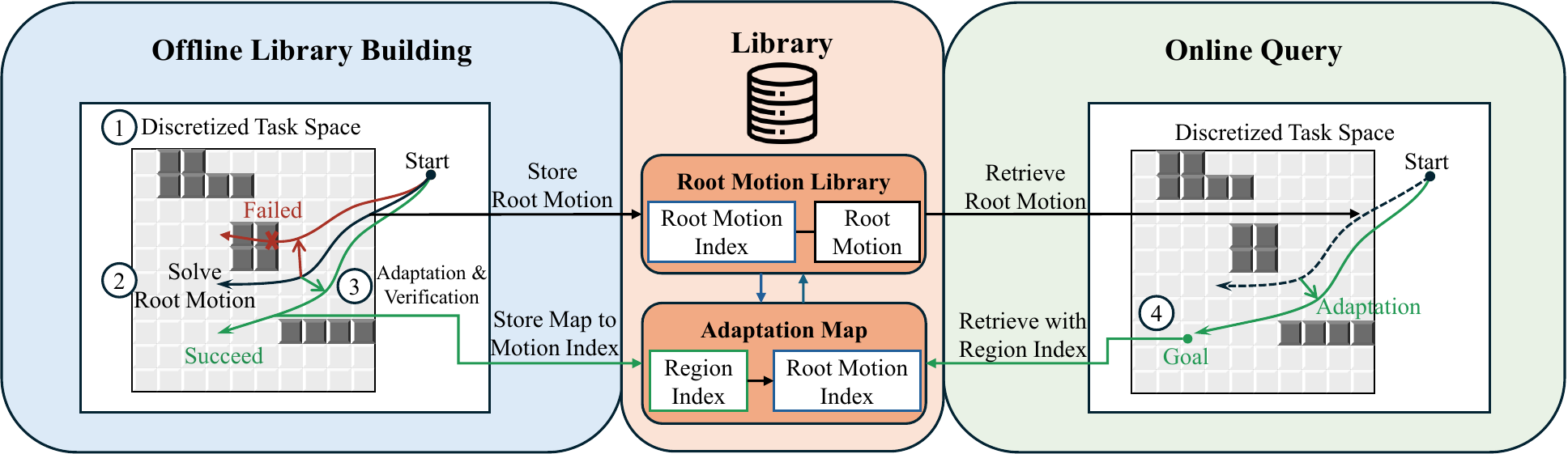}
    \caption{\method consists of an offline library-building stage and an online query stage. \textbf{1)} The continuous task space is discretized into a finite set of TCRs, each indexed for constant-time lookup. \textbf{2)} During offline preprocessing, an uncovered region is selected, a motion planner finds a \emph{root motion} for that region, and the resulting path is stored in the root-motion library. \textbf{3)} The same root motion is then adapted and verified on neighboring regions; successful adaptations are recorded in an adaptation map that links each covered region index to a root-motion index. \textbf{4)} During the online query stage, the sensed Goal pose is mapped to its region index, the corresponding root motion is retrieved in constant time, and the motion is adapted to satisfy the queried goal condition using an adaption algorithm. This process yields a compressed library that provides constant plans and covers the full task space.}
    \label{fig:method}
\vspace{-2ex}
\end{figure*}

\section{Notations and Problem Statement}
\label{sec:problem}

Let $q \in \C$ denote the configuration and configuration space of a robot. 
Let $\Cobs$ denote the obstacle (invalid) configuration space, and denote the collision-free space as $\Cfree=\C \setminus \Cobs$.

\subsection{Motion Planning Problem} 
Given a start configuration $\qs \in \Cfree$ and a goal region $\Cgoal \subseteq \Cfree$, motion planning problem tries to find a continuous path $\qtraj : [0,1] \rightarrow \Cfree$ such that $\qtraj(0)=\qs, \,\, \qtraj(1)\in \Cgoal$.

\subsection{Goal-varying motion planning problem}
In this work, we study a class of manipulation planning problems in which the goal depends on the pose of a movable object. 
Let $\T_o \in \Ts \subseteq SE(3)$ denote the pose of a target object $o$, whose placement varies across queries within the bounded task space $\Ts$ while the environment remains static. 
Such settings are common in tasks such as robotic grasping, where the robot must reach a grasp pose relative to the object (\autoref{fig:intro}).
For each object pose $\T_o$, the task induces a goal region in the robot
configuration space, denoted $\Cgoal(\T_o)$. Because the object itself occupies space in the workspace, its pose also affects the set of valid robot configurations. Consequently, the collision-free configuration space becomes object-dependent and is denoted $\Cfree(\T_o)$. 

Given an object pose $\T_o \in \mathcal{T}$, a start configuration $\qs \in \Cfree(\T_o)$ and a goal region $\Cgoal(\T_o)$, we aim to compute a continuous path
$\qtraj : [0,1] \rightarrow \Cfree(\T_o) $
such that $ \qtraj(0) = \qs,~\qtraj(1) \in \Cgoal(\T_o), $
if such a path exists.
We denote the family of these problems as: 
$$p(\Ts) = \{(\qs, \Cgoal(\T_o), \Cfree(\T_o)) | \T_o \in \Ts \}$$ 
where both the goal region and the collision-free space vary continuously with $\T_o$.

\section{Methodology}
Since the set of motion planning problems $p(\mathcal{T})$ contains infinitely many instances, computing a unique path for each is infeasible. Instead, we leverage the observation that a single path can remain valid for a continuous subset of object poses.

In \autoref{sec:methodology-tsr}, we show how this property allows us to
partition the task space into a finite set of regions. In
\autoref{sec:methodology-library} and \autoref{sec:methodology-adapt}, we describe how \method leverages these finite regions to construct a motion-plan library. Finally, we show in \autoref{sec:methodology-online} that it enables constant-time solution for any problem $\T_o \in \Ts$. The overall framework is illustrated in \autoref{fig:method}.

\subsection{Task Space Discretization}
\label{sec:methodology-tsr}

\subsubsection{Task Space Region}
We begin by describing the goal regions $\Cgoal(\T_o)$ induced by the pose of the object using the \emph{Task Space Regions} (TSRs)~\cite{tsr}.
A TSR defines a set of admissible end-effector poses relative to an object that are valid for completing the task (e.g., grasping), as in \autoref{fig:proof} (a).

Formally, let $\T_\delta \in \SEThree$ denote a displacement transform bounded by a box $\mathbf{B} \subset \mathbb{R}^6$ specifying translational and rotational limits. 
The resulting displacement set is denoted $\Delta(\mathbf{B})$.
Let $\T_{s} \in \SEThree$ denote a fixed end-effector offset relative to the object frame, and let $\T_o$ denote the pose of object $o$.
The corresponding TSR is defined as:

\begin{equation}
\label{eq:tsr}
\mathrm{TSR}(\T_o)
=
\{ \T_e \mid \exists \T_\delta \in \Delta(\mathbf{B}) 
\text{ s.t. } \T_e = \T_o \T_\delta \T_{s} \}.
\end{equation}

Geometrically, the TSR is generated by allowing the local offset pose \(\T_\delta\) to vary over a \(6\)-dimensional box, and mapping each such offset to an end-effector pose via rigid body transformations \(\T_e = \T_o \T_\delta \T_{s}\). Thus, in the local coordinates of \(\T_\delta\), the TSR is exactly a box-shaped region.

The goal region in configuration space is therefore:
\begin{equation}
\Cgoal(\T_o) =
\{ q \mid \fk(q) \in \mathrm{TSR}(\T_o) \},
\end{equation}
where $\fk$ denotes the forward kinematics of the robot.

\subsubsection{Shared goals across object poses}
Consider two sufficiently close object poses $\T_{o1}$ and $\T_{o2}$ whose TSRs overlap,
as illustrated in \autoref{fig:proof} (b).
In this case, there exists an end-effector pose $\T_e$ that lies in the intersection of the two TSRs and therefore constitutes a valid goal for both object placements.

More importantly, in the example shown in \autoref{fig:proof} (c), the shared end-effector pose remains valid throughout the illustrated transformation between $\T_{o1}$ and $\T_{o2}$. This observation implies that a single end-effector goal can be valid  for continuous subset of object poses.

\subsubsection{Task Coverage Regions}
Motivated by this observation, we introduce the notion of a Task Coverage Region (TCR).
Rather than associating each object pose with a set of valid end-effector poses (as in TSRs), we invert this relationship and define the set of object poses that a fixed end-effector pose can solve.

\begin{definition}[Task Coverage Region]
Given a TSR with  $\T_{\delta}$, $\Delta(\mathbf{B})$ and $\T_{s}$, we define the Task Coverage Region (TCR) of an end-effector pose $\T_e$ as the set of object poses for which $\T_e$ lies within their TSRs:
\begin{equation}
\label{eq:tcr}
\mathrm{TCR}(\T_e)
=
\{ \T_o \mid \T_e \in \mathrm{TSR}(\T_o) \}.
\end{equation}
Substituting the TSR definition from \autoref{eq:tsr} yields
\begin{equation}
\label{eq:tcr2}
\mathrm{TCR}(\T_e)
=
\{ \T_e \T_{s}^{-1} \T_\delta^{-1} \mid \T_\delta \in \Delta(\mathbf{B}) \}.
\end{equation}
\end{definition}

\begin{figure}[ht]
    \centering
    \includegraphics[width=\linewidth]{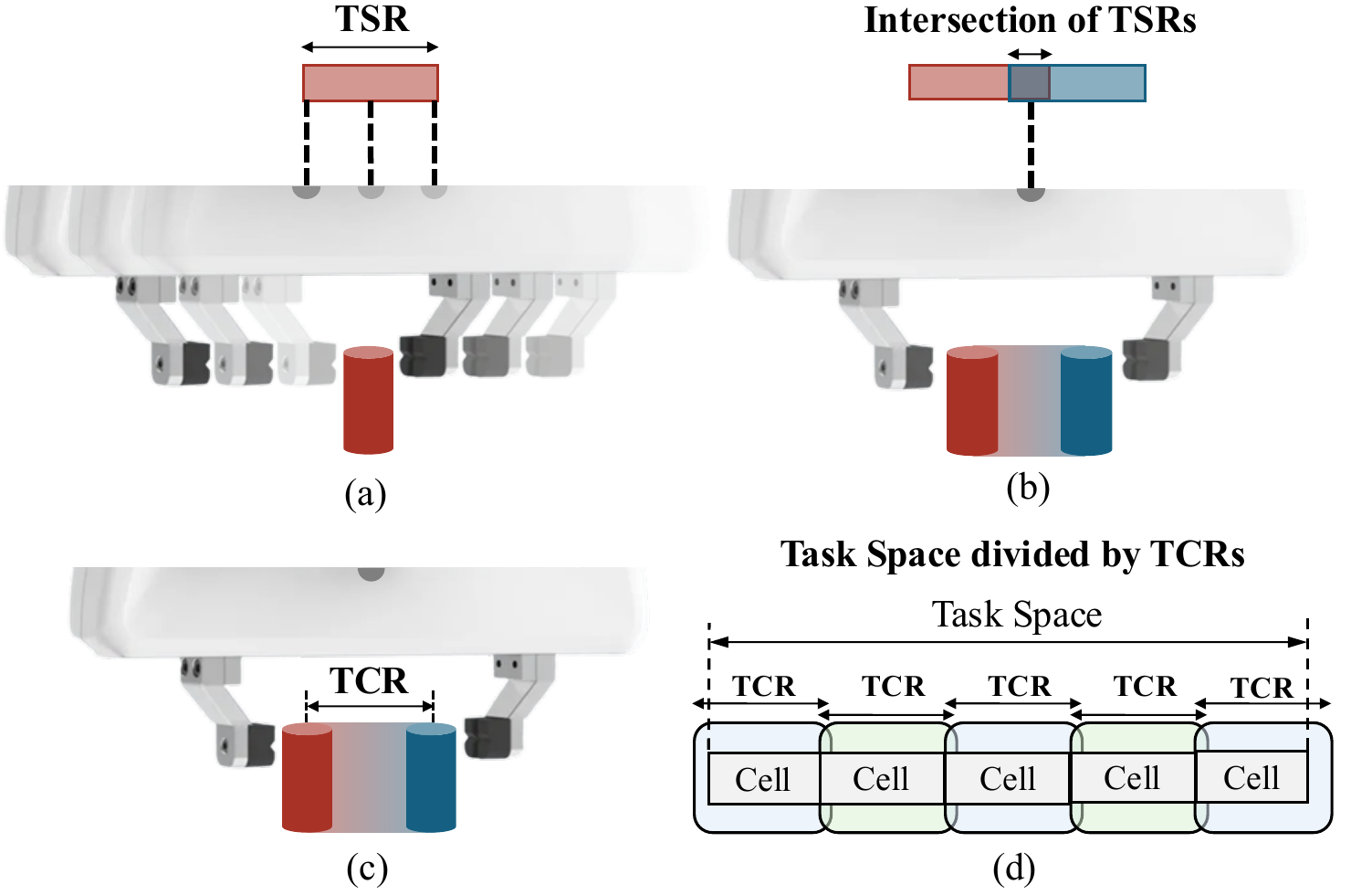}
    \caption{1D illustration of TSRs and TCRs.} 
    \label{fig:proof}
\vspace{-2ex}
\end{figure}

Intuitively, a TCR describes the region of object poses in \(\mathcal{T}\) for which a fixed end-effector pose \(\T_e\) remains valid.
Geometrically, a TCR is obtained by allowing the local offset pose \(\T_\delta\) to vary over the same \(6\)-dimensional box and solving for the corresponding object pose through: \mbox{$\T_o = \T_e \T_{s}^{-1} \T_\delta^{-1}$}. However, due to the inversion this region is not an axis-aligned box in the object-pose local coordinates of~\(\mathcal{T}\).

\paragraph{Finite coverage of the task space}
Our goal is to represent the continuous task space $\Ts$ using a finite set of representative end-effector goals.
Specifically, we seek a set of end effector poses $\{\T_e^i\}_{i=1}^{N}$ such that every task pose $\T_o \in \Ts$ lies within at least one coverage region:
\begin{equation}
\label{eq:cover}
\Ts \subseteq \bigcup_{i=1}^{N} \mathrm{TCR}(\T_e^i).
\end{equation}
Each representative end-effector pose $\T_e^i$ is therefore valid for every object pose contained in its corresponding TCR.

Since TCRs may have irregular shapes and may overlap, we discretize the task space $\Ts$ into axis-aligned cells to enable constant-time indexing. Each cell must lie entirely within one TCR so that a single representative end-effector pose $\T_e$ is valid for every task pose in that cell, as illustrated in \autoref{fig:proof} (d).

We first establish that such non-degenerate local cells exist.
Specifically, whenever the TSR displacement set $\Delta(\mathbf{B})$ has non-empty interior, every TCR also has non-empty interior and therefore contains an axis-aligned box under a local task-space coordinate chart.
\begin{theorem}[Local TCR coverage]
\label{thm:inner-box-se3}
Let $\Delta(\mathbf{B}) \subset SE(3)$ be the bounded displacement set
of a TSR. If $\Delta(\mathbf{B})$ has non-empty interior, then for any
fixed $\T_e,\T_s \in SE(3)$, $\mathrm{TCR}(\T_e)$ also has non-empty
interior. Consequently, under any valid local coordinate chart,
$\mathrm{TCR}(\T_e)$ contains a non-degenerate axis-aligned box
$\mathbf{B}_{\mathrm{in}}$.
\end{theorem}

\begin{proof}[Proof]
From \autoref{eq:tcr2}, define the map $f(\T) = \T_e \T_s^{-1}\T^{-1}$.
Matrix inversion and left multiplication by a fixed rigid transform are
diffeomorphisms on $SE(3)$, hence $f$ is also a diffeomorphism and therefore maps open sets to open sets.
Since $\Delta(\mathbf{B})$ has non-empty interior, it contains a non-empty open set $U$. Then $f(U)$ is open and satisfies $f(U) \subseteq \mathrm{TCR}(\T_e)$.
Therefore, $\mathrm{TCR}(\T_e)$ contains a non-empty open set and has
non-empty interior.
Under any valid local task-space coordinate chart, this open set corresponds to a non-empty open subset of $\R^6$, which contains an axis-aligned box with non-empty interior.
\end{proof}
~\autoref{thm:inner-box-se3} establishes local TCR coverage. If the identity transform
$\mathbf{I}$ lies in the interior of $\Delta(\mathbf{B})$, then for any task
pose $\T_o$, choosing $\T_e=\T_o\T_s$ gives
$\T_o=\T_e\T_s^{-1}\mathbf{I}^{-1}$, so $\T_o$ lies in the interior of
$\mathrm{TCR}(\T_e)$. Hence the TCRs form an open cover of the task space,
and for a compact task domain, a finite subcover exists.

Although such inner regions always exist, deriving an explicit axis-aligned box for general $SE(3)$ variations is difficult because inversion couples translation and rotation. We therefore first consider the most common manipulation setting in which objects vary in $(x,y,z,\psi)$, where $\psi$ is yaw, and roll and pitch remain fixed.

Let the TSR bounds be
$|\Delta x|\le b_x$, $|\Delta y|\le b_y$,
$|\Delta z|\le b_z$, and $|\Delta\psi|\le b_\psi$.
After inversion, planar translation is coupled with yaw. A conservative
yaw-independent inner bound is obtained by
\begin{equation}
\label{eq:hxy}
h_{xy}=\frac{\min(b_x,b_y)}{\sqrt{2}},
\end{equation}
since any point satisfying
$|\Delta x|,|\Delta y|\le h_{xy}$ has planar norm no greater than
$\min(b_x,b_y)$ and therefore remains inside the admissible translation
region under any yaw rotation. The $z$ and yaw components remain decoupled,
yielding the inner-box widths:
\begin{equation}
\label{eq:cellwidth}
\delta_x=2h_{xy},\quad
\delta_y=2h_{xy},\quad
\delta_z=2b_z,\quad
\delta_\psi=2b_\psi.
\end{equation}

These widths define a regular grid over the bounded task space. For
$x\in[x_l,x_h]$, $y\in[y_l,y_h]$, $z\in[z_l,z_h]$, and wrapped yaw
$\bar{\psi}\in[0,2\pi)$, a query pose is indexed by
\begin{equation}
\label{eq:tcr-index}
i_x=\lfloor \frac{x-x_{l}}{\delta_x}\rfloor,\,
i_y=\lfloor \frac{y-y_{l}}{\delta_y}\rfloor,\,
i_z=\lfloor \frac{z-z_{l}}{\delta_z}\rfloor,\,
i_\psi=\lfloor \frac{\bar{\psi}}{\delta_\psi}\rfloor.
\end{equation}
For each cell $r$, let $\T_o^r$ be its center and define
$\T_e^r=\T_o^r\T_s$. By construction, the entire cell lies within
$\mathrm{TCR}(\T_e^r)$, and the index tuple uniquely indexes a TCR cell and can be used as a key to retrieve end-effector goals and motion solutions.

For more general task variations, including full $SE(3)$ poses or additional articulated-object parameters, we use a heuristic yet practical numerical procedure to determine candidate inner-box dimensions. Starting from a nominal pose $\T_o$ with $\T_e=\T_o\T_s$, we perturb the object pose along each task-space dimension and sample poses within the resulting candidate region to evaluate whether $\T_e$ remains inside the corresponding TSR.

For every constructed cell, we additionally check the associated robot goal configuration against a conservative enclosure of the goal-object geometry over the entire cell. If this check fails, the cell bounds are reduced until the collision constraint is satisfied.

\subsection{Offline Library Building with Compression}
\label{sec:methodology-library}

\begin{algorithm}
\caption{$\mathtt{BuildLibrary}(\mathcal{R}, \adapt, \qs)$}
\label{alg:library}
\SetNoFillComment

\tcc{Initialize library}

$\library \gets$ Empty root motion list\;
\nllabel{lib:init_library}

$\map \gets$ Empty hash map from a TCR $r$ to root motion $\trajrep_\rt$ and joint goal of the region $q$\;
\nllabel{lib:init_map}

\While{$\mathcal{R} \neq \emptyset$}{
    \nllabel{lib:while}

    \tcc{Sample a new root region}
    $r \gets \mathtt{UniformSampling}(\mathcal{R})$\;
    \nllabel{lib:sample_root}

    $\mathcal{R}.\mathtt{Remove}(r)$\;
    \nllabel{lib:remove_root}

    \tcc{Obtain a root goal in joint space}
    $q_\rt \gets \ik(r, q_\start)$\;
    \nllabel{lib:ik_root}

    \If{$q_\rt = null$}{
        \nllabel{lib:ik_root_fail_start}
        \Continue\;
        \nllabel{lib:ik_root_fail_end}
    }

    \tcc{Solve for a root path}
    $\qtraj_\rt \gets \mathtt{Plan}(q_\start, q_\rt, \Cfree(r))$\;
    \nllabel{lib:plan_root}

    \If{$\qtraj_\rt = null$}{
        \nllabel{lib:plan_root_fail_start}
        \Continue\;
        \nllabel{lib:plan_root_fail_end}
    }


    \tcc{Store the root motion and its index}
    $i_{\rt} \leftarrow \library.\mathtt{Append}(\trajrep_\rt)$
    \nllabel{lib:append_library}

    \tcc{Register the root region as covered by this root motion}
    $\map(r) \gets i_\rt, q_\rt$\;
    \nllabel{lib:map_root}

    \tcc{Try to cover nearby regions via adaptation from the same root motion}
    \ForAll{$r_n \in \mathtt{NearestNeighbor}(r, N_\text{neighbor})$}{
        \nllabel{lib:nn_loop}

        $q_n \gets \ik(r_n, q_\rt)$\;
        \nllabel{lib:ik_neighbor}

        \If{$q_n = null$}{
            \nllabel{lib:ik_neighbor_fail_start}
            \Continue\;
            \nllabel{lib:ik_neighbor_fail_end}
        }

        $\qtraj_n \gets \adapt.\mathtt{Adapt}(\trajrep_\rt, q_n)$\;
        \nllabel{lib:adapt_neighbor}

        \If{$\mathtt{IsValidPath}(\qtraj_n) = true$}{
            \nllabel{lib:valid_start}

            $\map(r_n) \gets i_\rt, q_n$\;
            \nllabel{lib:map_neighbor}

            $\mathcal{R}.\mathtt{Remove}(r_n)$\;
            \nllabel{lib:remove_neighbor}

            \nllabel{lib:valid_end}
        }
    }
}
\Return $\library, \map$\;
\nllabel{lib:return}

\end{algorithm}

After task space discretization, a brute-force library would simply plan and store a path for every TCR. This is often impractical due to its large memory footprint.

In line with library compression strategies\cite{ishida2025coverlib, constant_belt}, we compress our library by planning root paths for a subset of TCRs and adapting them to cover additional regions, as summarized in \autoref{alg:library}.
The algorithm  takes TCRs $\mathcal{R}$, adaptation method $\adapt$ and starting configuration $\qs$ as input. It initializes an empty root-motion list $\library$ and an empty hash map $\map$ that associates each region $r$ with a root motion $\Pi_\rt$ and the joint goal of the region $q$ (\autoref{lib:init_library}-\autoref{lib:init_map}).
While $\mathcal{R}$ is not empty and we have uncovered regions, the algorithm uniformly samples a region $r$ from the TCRs $\mathcal{R}$ as root region and removes it from further consideration (\autoref{lib:while}-\autoref{lib:remove_root}). 
It then solves inverse kinematics (IK), warm-started from $\qs$, to obtain a root goal configuration $q_\rt$ for the task at $r$ (\autoref{lib:ik_root}). If $\ik$ fails, the iteration terminates and continues to the next sample (\autoref{lib:ik_root_fail_start}-\autoref{lib:ik_root_fail_end}).

Given a feasible IK goal, a root joint-space path $\qtraj_\rt$ is planned from $q_\start$ to $q_\rt$ with free space $\Cfree(r)$ conditioned on the TCR $r$ (\autoref{lib:plan_root}), where collisions account for the swept volume induced by all possible goal poses within $r$~\cite{cover}.
In this work, we use RRT-Connect~\cite{rrt-connect} to plan the initial solution and run path simplification and smoothing~\cite{ompl}.
If planning fails, the task is considered infeasible and the algorithm continues to the next region (\autoref{lib:plan_root_fail_start}-\autoref{lib:plan_root_fail_end}).
Otherwise, the resulting root path is added to the library $\library$ (\autoref{lib:append_library}).
The root region $r$ is then covered by $\trajrep_\rt$, conditioned on $q_\rt$, and being registered in the map (\autoref{lib:map_root}).

Finally, the algorithm attempts to expand region coverage by iterating over a set of $N_\text{neighbor}$ nearest neighbors $r_n$ around the root region (\autoref{lib:nn_loop}). In this work, we choose $N_\text{neighbor} = 1000$.
For each neighbor, it solves IK, warm-started from $q_\rt$, to obtain a joint goal $q_n$ of the neighbor region (\autoref{lib:ik_neighbor}-\autoref{lib:ik_neighbor_fail_end}). If successful, the adaptation module $\adapt$ generates a candidate path $\qtraj_n$ by transforming the stored root motion $\trajrep_\rt$ conditioned on $q_n$ with method $\mathtt{Adapt}$ (\autoref{lib:adapt_neighbor}).
If the adapted path can both reach the goal and is collision free, the neighbor region is marked as covered in $\map$ and removed from $\mathcal{R}$ (\autoref{lib:valid_start}-\autoref{lib:remove_neighbor}).
When $\mathcal{R}$ becomes empty, it returns the library $\library$ and the region-to-solution map $\map$ (\autoref{lib:return}).
%

\subsection{Adaptation Methods}
\label{sec:methodology-adapt}

%
All adaptations in \method are designed to be fast and deterministic.
To preserve efficiency, we propose to deliberately \emph{exclude} obstacle terms from the adaptation objectives. Instead, collision checking is performed during offline verification (\autoref{lib:valid_start} of \autoref{alg:library}), ensuring that the adapted path is collision-free.
Below we evaluate three different adaptation strategies, each in terms of its general idea, what is stored as $\trajrep_\rt$, and how $\mathtt{Adapt}$ produces path $\qtraj$.

\subsubsection{Linear Interpolation (LI)}
\label{sec:methodology-adapt-interpolation}
For nearby goals, the simplest adaption strategy is to connect the end of the root path $q_\rt$ with the desired goal $q_\goal$ through linear interpolation. Specifically, we construct a waypoint sequence $\qtraj_{\text{end}} = \{q_0, \dots, q_{T_\text{end}}\}$ of length $T_\text{end}$,
whose first configuration is $q_0 = q_\rt$ and whose last configuration is $q_{T_\text{end}} = q_\goal$:
\begin{equation}
\label{eq:interpolation}
q_n =
(1 - \tfrac{n}{T_{\text{end}}}) q_{\rt}
+
\tfrac{n}{T_{\text{end}}} q_{\goal},
\,\, n\in[1, T_\text{end}],
\end{equation}
and append $\qtraj_\text{end}$ to the root path $\qtraj_\rt$. This produces a joint-space path that reaches $q_\goal$ exactly. $T_\text{end}$ can be chosen with the desired path resolution. Also, we include a continuity verification from~\cite{expansion-grr} to avoid end-effector discontinuities.
%

\subsubsection{Dynamic Movement Primitives (DMPs)}
\label{sec:methodology-adapt-dmp}
DMPs provide a compact dynamical system representation of a demonstrated path and allow fast goal retargeting while preserving qualitative shape~\cite{dmp_survey}.
Given a path as reference, we model each joint path using a second-order stable dynamical system with a learned nonlinear forcing term:
\begin{equation}
\label{eq:dmp}
\ddot{q}_i
=
\alpha_y \left( \beta_y ( g_i - q_i ) - \dot{q}_i \right) + f_i(s),
\end{equation}
where $\alpha_y, \beta_y > 0$ are gain hyper-parameters controlling convergence and damping, $g_i$ is the goal of each joint, $i$ denotes the joint index, and $s \in [0,1]$ is a phase variable for motion progression.
The forcing term $f_i(s)$ is represented as a weighted sum of $N_B$ fixed basis functions and is learned from the reference path.
%
%

For this method, instead of storing $\trajrep_\rt$ we fit independent DMPs with $\qtraj_\rt$ to each joint dimension and instead store the learned basis weights as the root motion representation $\trajrep_\rt$. 
In contrast to storing all $T$ waypoints, DMPs require only $O(N_B)$ parameters per joint.
The function $\adapt.\mathtt{Adapt}$ performs a forward rollout of~\autoref{eq:dmp} with the goal set to the given $q_\goal$, producing the adapted trajectory $\qtraj$. 
%

\subsubsection{Simple Trajectory Optimization (STO)}
\label{sec:methodology-adapt-opt}
The STO adaptation performs a simplified trajectory refinement inspired by trajectory optimization via sequential convex optimization (TrajOpt)~\cite{trajopt}, and by the observation that stored motions can serve as effective warm-starts for optimization~\cite{memory_of_motion}.
Given the stored root path $\qtraj_\rt$ as initialization, we solve a single convex quadratic program that minimizes joint-space velocity and acceleration while penalizing deviation from the stored root path, subject to boundary and joint-limit constraints. Let $q_{0:T} := (q_0,\dots,q_T)$ denote the path variables, and let $q_{\rt, 0:T}$ denote the waypoints of the root path. The QP~\cite{osqp} solves:
\begin{equation}
\label{eq:opt}
\begin{gathered}
\min_{q_{0:T}}\,
w_{\mathrm{v}} \| D_1 \, q_{0:T} \|_2^2
+
w_{\mathrm{a}} \| D_2 \, q_{0:T} \|_2^2
+
w_{\mathrm{s}} \| q_{0:T} - q_{\rt, 0:T} \|_2^2 \\
\text{s.t.}\quad
q_0 = q_\start,\,\,
q_T = q_\goal,\,\,
q_{\min} \le q_t \le q_{\max}
\,\, \forall t ,
\end{gathered}
\end{equation}
where $w_{\mathrm{v}}, w_{\mathrm{a}}, w_{\mathrm{s}} > 0$ are weighting parameters, $D_1$ and $D_2$ denote first- and second-order finite-difference operators along the path, 
The function $\adapt.\mathtt{Adapt}$ solves the convex QP in \autoref{eq:opt}, and returns an optimized path $\qtraj$ satisfying the desired new goal $q_\goal$.
%
\subsection{Online Queries}
\label{sec:methodology-online}

\begin{algorithm}
\caption{$\mathtt{Query}(\T_o, \map, \library, \adapt)$}
\label{alg:online}
\SetNoFillComment

\tcc{Find TCR region}
$r \gets \mathtt{GetRegion}(\T_o)$\;
\nllabel{query:index}

\tcc{Acquire joint goal and root motion}
\If{$r \notin \map$}{
    \nllabel{query:map_check_start}
    \Return $null$\;
    \nllabel{query:map_check_end}
}

$i_\rt, q_\goal \gets \map(r)$\;
$\trajrep_\rt \gets \library(i_\rt)$\;
\nllabel{query:map}

\tcc{Apply adaptation}
$\qtraj_\text{sol} \gets \adapt.\mathtt{Adapt}(\trajrep_\rt, q_\goal)$\;
\nllabel{query:adapt}

\Return $\qtraj_\text{sol}$\;
\nllabel{query:return}

\end{algorithm}

\method answers an online query by retrieving and adapting a stored root motion from the compressed library, as summarized in \autoref{alg:online}. 
Given a sensed object pose $\T_o \in \Ts$, the procedure first computes its associated TCR index (\autoref{eq:tcr-index}) to find TCR region $r$ (\autoref{query:index}).
If the index is not covered by the map, the query terminates and returns null, indicating that solution is not available for this problem (\autoref{query:map_check_start} - \autoref{query:map_check_end}).
Otherwise, the map returns the corresponding root motion $\trajrep_\rt$ together with the goal joint configuration $q_\goal$ for this problem (\autoref{query:map}).
Finally, the adaptation module $\adapt$ transforms $\trajrep_\rt$ into a joint-space path conditioned on $q_\goal$, producing the solution $\qtraj_\text{sol}$ (\autoref{query:adapt}-\autoref{query:return}).

The online phase therefore reduces planning to a constant-time retrieval followed by lightweight goal-conditioned deformation. The online cost is dominated by two components: (i) TCR indexing and retrieval, and (ii) adaptation cost.

As shown in \autoref{eq:tcr-index}, the acquisition of TCR index and TCR region $r$ is a fixed $O(1)$ cost arithmetic operation, and the retrieval of $\trajrep_\rt$ from $\map$ and $\library$ are expected $O(1)$ with hash lookup and array indexing.

Among the adaptation methods, linear interpolation is $O(T_\text{end})$, where $T_\text{end}$ is the number of interpolated waypoints from the root end to the desired goal. In this work, we use $T_\text{end}=10$.
The time complexity of DMP rollout is $O(T \cdot N_B)$ with fixed $B$ basis functions and $T$ generated waypoints.
For the optimization-based adapter STO, the runtime scales approximately as $O(T \cdot I)$, where $I$ denotes the max number of iterations.
In motion planning practice, where trajectories are stored at a fixed resolution (in this work, $T = 200$), this yields a predictable fixed runtime.
Since all adapter parameters, such as $T_\text{end}, T, N_B$ and $I$, are fixed offline, the online query time is constant with respect to the online input.

%
%

\section{Experiments}
\label{sec:experiment}

\begin{figure}[ht]
    \centering
    \includegraphics[width=0.98\linewidth]{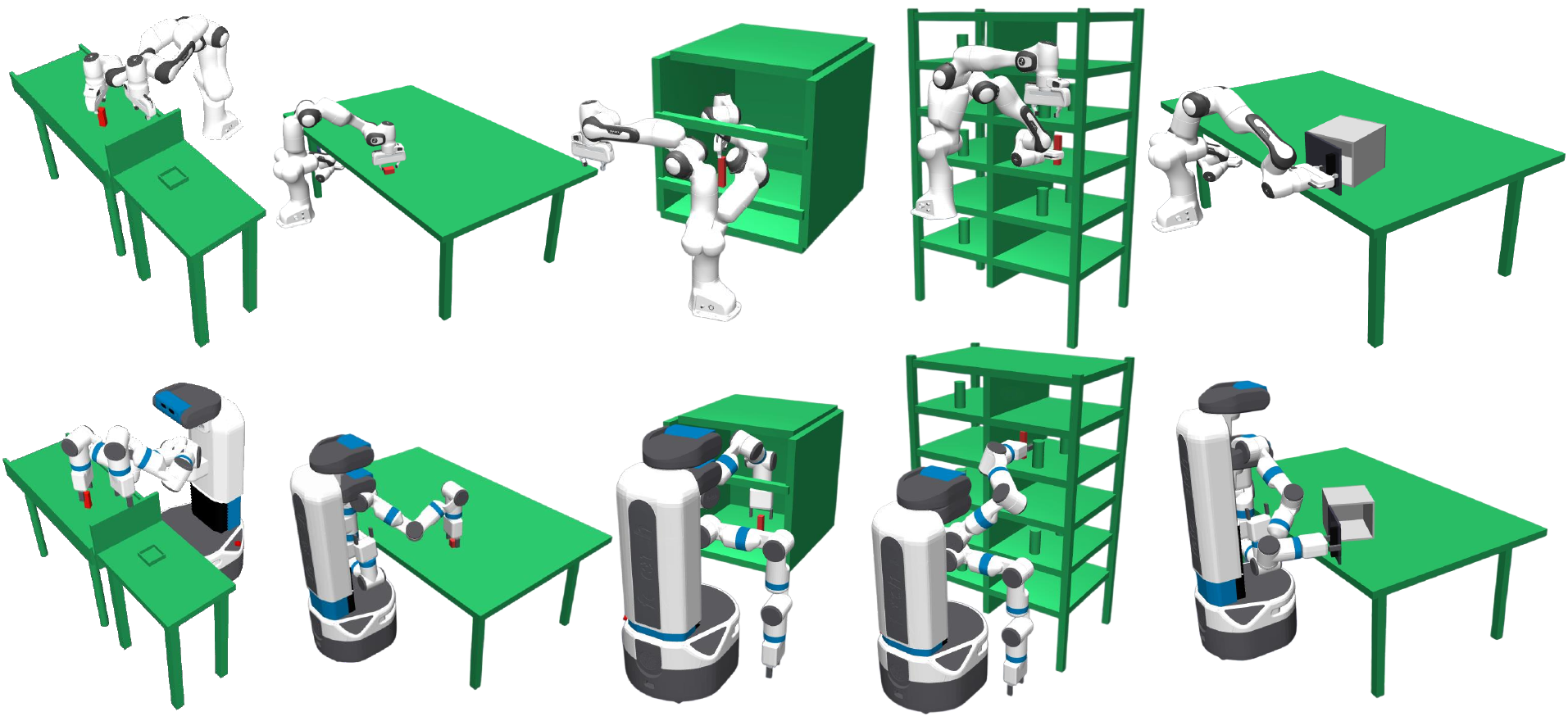}
    
    \caption{Panda (top) and Fetch (bottom) manipulators planning in Conveyor, All Stable, Cage, Shelf, Microwave environments (from left to right)\cite{chamzas2022-motion-bench-maker}. } 
    \label{fig:exp_env}
\vspace{-2ex}
\end{figure}

\begin{table}[t!]
\caption{Plan Library (OOM denotes Out of Memory)}
\label{tab:library}
\begin{NiceTabular}{c|c|c|[tikz=dashed]ccc}
\hline
  \textbf{\begin{tabular}[c]{@{}c@{}}Robot \\ -Env.\end{tabular}} &
  \textbf{Metric} &
  \textbf{\begin{tabular}[c]{@{}c@{}}Full \\ Library\end{tabular}} &
  \textbf{\begin{tabular}[c]{@{}c@{}}$\method$\\ -LI\end{tabular}} &
  \textbf{\begin{tabular}[c]{@{}c@{}}$\method$\\ -DMP\end{tabular}} &
  \textbf{\begin{tabular}[c]{@{}c@{}}$\method$\\ -STO\end{tabular}} \\ \hline
\multirow{4}{*}{\begin{tabular}[c]{@{}c@{}}Panda\\ -Conveyor\end{tabular}}      & Comp. (\%)    & -       & 98.72          & \textbf{99.41} & 98.73          \\
                                                                             & Time (ms)     & 0.042   & \textbf{0.053} & 0.894          & 3.175          \\
                                                                             & Quality (rad) & 1.47    & 2.18           & \textbf{1.53}  & 2.14           \\
                                                                             & Size          & 13974    & 179           & 82            & 178           \\ \hline
\multirow{4}{*}{\begin{tabular}[c]{@{}c@{}}Panda\\ -All Stable\end{tabular}} & Comp. (\%)    & -       & 74.39          & 34.10          & \textbf{93.82} \\
                                                                             & Time (ms)     & 0.039   & \textbf{0.058} & 1.150          & 5.493          \\
                                                                             & Quality (rad) & 5.31    & 5.48           & \textbf{5.18}  & 5.42           \\
                                                                             & Size          & 74107   & 18982          & 48840          & 4580           \\ \hline
\multirow{4}{*}{\begin{tabular}[c]{@{}c@{}}Panda\\ -Cage\end{tabular}}       & Comp. (\%)    & -       & 78.37          & 63.41          & \textbf{80.83} \\
                                                                             & Time (ms)     & 0.035   & \textbf{0.054} & 1.215          & 4.322          \\
                                                                             & Quality (rad) & 7.93    & 7.95           & \textbf{7.26}  & 8.03           \\
                                                                             & Size          & 6693    & 1448           & 2449           & 1283           \\ \hline
\multirow{4}{*}{\begin{tabular}[c]{@{}c@{}}Panda\\ -Shelf\end{tabular}}      & Comp. (\%)    & -       & 79.03          & 56.69          & \textbf{82.19} \\
                                                                             & Time (ms)     & 0.039   & \textbf{0.058} & 1.240          & 4.483          \\
                                                                             & Quality (rad) & 4.78    & 5.03           & \textbf{4.83}  & 4.95           \\
                                                                             & Size          & 23290   & 4885           & 10088          & 4148           \\ \hline
\multirow{4}{*}{\begin{tabular}[c]{@{}c@{}}Panda\\ -Microwave\end{tabular}}  & Comp. (\%)    & -       & 91.37          & OOM               & \textbf{94.99} \\
                                                                             & Time (ms)     & OOM       & \textbf{0.084} & OOM              & 6.186          \\
                                                                             & Quality (rad) & OOM       & 5.95           & OOM              & \textbf{5.94}  \\
                                                                             & Size          & 6414010 & 553801         & OOM               & 320894         \\ \hline
\multirow{4}{*}{\begin{tabular}[c]{@{}c@{}}Fetch\\ -Conveyor\end{tabular}}      & Comp. (\%)    & -       & 98.35          & 98.33 & \textbf{98.41}          \\
                                                                             & Time (ms)     & 0.043   & \textbf{0.054} & 0.949          & 3.538          \\
                                                                             & Quality (rad) & 3.78    & 4.39           & \textbf{2.77}  & 4.00           \\
                                                                             & Size          & 10177    & 168           & 170            & 162           \\ \hline
\multirow{4}{*}{\begin{tabular}[c]{@{}c@{}}Fetch\\ -All Stable\end{tabular}} & Comp. (\%)    & -       & 62.77          & 9.58           & \textbf{71.29} \\
                                                                             & Time (ms)     & 0.042   & \textbf{0.062} & 1.556          & 5.387          \\
                                                                             & Quality (rad) & 11.78   & 11.83          & \textbf{11.58} & 11.75          \\
                                                                             & Size          & 38480   & 14327          & 34793          & 11046          \\ \hline
\multirow{4}{*}{\begin{tabular}[c]{@{}c@{}}Fetch\\ -Cage\end{tabular}}       & Comp. (\%)    & -       & 65.89          & 38.21          & \textbf{72.91} \\
                                                                             & Time (ms)     & 0.043   & \textbf{0.066} & 3.761          & 10.134         \\
                                                                             & Quality (rad) & 18.0    & 17.11          & \textbf{16.27} & 17.30          \\
                                                                             & Size          & 683     & 233            & 422            & 185            \\ \hline
\multirow{4}{*}{\begin{tabular}[c]{@{}c@{}}Fetch\\ -Shelf\end{tabular}}      & Comp. (\%)    & -       & 81.74          & 53.60          & \textbf{83.67} \\
                                                                             & Time (ms)     & 0.041   & \textbf{0.061} & 1.333          & 4.395          \\
                                                                             & Quality (rad) & 7.68    & 7.69           & \textbf{7.14}  & 7.56           \\
                                                                             & Size          & 36373   & 6641           & 16876          & 5938           \\ \hline
\multirow{4}{*}{\begin{tabular}[c]{@{}c@{}}Fetch\\ -Microwave\end{tabular}}  & Comp. (\%)    & -       & 78.08          & OOM               & \textbf{80.62} \\
                                                                             & Time (ms)     & OOM       & \textbf{0.082} & OOM              & 9.535          \\
                                                                             & Quality (rad) & OOM       & 16.29          & OOM              & \textbf{16.25} \\
                                                                             & Size          & 929441  & 203716         & OOM               & 180082         \\ \hline
\multirow{4}{*}{\begin{tabular}[c]{@{}c@{}}UR10\\ -Real\end{tabular}}        & Comp. (\%)    & -       & 93.54          & 69.99          & \textbf{95.22} \\
                                                                             & Time (ms)     & 0.042   & \textbf{0.061} & 1.502          & 4.615          \\
                                                                             & Quality (rad) & 6.43    & \textbf{6.55}  & 6.61           & 6.64           \\
                                                                             & Size          & 4415    & 285            & 1325           & 211            \\ \hline
\end{NiceTabular}
\vspace{-2ex}
\end{table}

\begin{table*}[ht]
\centering
\caption{Path Quality and Success Rate}
\label{tab:result}

\begin{tabular}{c|c|cccccc}
\hline
Robot - Env. &
  Metric &
  \begin{tabular}[c]{@{}c@{}}RRT-Connect\\ \end{tabular} &
  \begin{tabular}[c]{@{}c@{}}RRT-Connect\\ -VAMP\end{tabular} &
  \begin{tabular}[c]{@{}c@{}}ERT-Connect\\ \end{tabular} &
  \begin{tabular}[c]{@{}c@{}}$\method$\\ -LI\end{tabular} &
  \begin{tabular}[c]{@{}c@{}}$\method$\\ -DMP\end{tabular} &
  \begin{tabular}[c]{@{}c@{}}$\method$\\ -STO\end{tabular} \\ \hline
\multirow{2}{*}{Panda - Conveyor} &
  Success (\%) &
  $\mathbf{100.00}$ &
  $\mathbf{100.00}$ &
  $\mathbf{100.00}$ &
  $\mathbf{100.00}$ &
  $\mathbf{100.00}$ &
  $\mathbf{100.00}$ \\
 &
  Quality (rad) &
  $5.36 \pm 0.80$ &
  $1.90 \pm 1.18$ &
  $2.65 \pm 1.88$ &
  $2.18 \pm 0.72$ &
  $\mathbf{1.53 \pm 0.57}$ &
  $2.15 \pm 0.70$ \\ \hline
\multirow{2}{*}{Panda - All Stable} &
  Success (\%) &
  $\mathbf{100.00}$ &
  $\mathbf{100.00}$ &
  $\mathbf{100.00}$ &
  $\mathbf{100.00}$ &
  $\mathbf{100.00}$ &
  $\mathbf{100.00}$ \\
 &
  Quality (rad) &
  $8.97 \pm 2.84$ &
  $8.53 \pm 1.33$ &
  $5.77 \pm 1.47$ &
  $5.42 \pm 1.14$ &
  $\mathbf{5.17 \pm 1.04}$ &
  $5.42 \pm 1.13$ \\ \hline
\multirow{2}{*}{Panda - Cage} &
  Success (\%) &
  $71.40$ &
  $99.90$ &
  $82.80$ &
  $\mathbf{100.00}$ &
  $\mathbf{100.00}$ &
  $\mathbf{100.00}$ \\
 &
  Quality (rad) &
  $20.50 \pm 6.55$ &
  $17.65 \pm 6.50$ &
  $9.39 \pm 2.80$ &
  $8.00 \pm 2.26$ &
  $\mathbf{7.27 \pm 2.04}$ &
  $7.91 \pm 2.26$ \\ \hline
\multirow{2}{*}{Panda - Shelf} &
  Success (\%) &
  $94.50$ &
  $\mathbf{100.00}$ &
  $98.80$ &
  $\mathbf{100.00}$ &
  $\mathbf{100.00}$ &
  $\mathbf{100.00}$ \\
 &
  Quality (rad) &
  $11.91 \pm 5.57$ &
  $9.13 \pm 3.87$ &
  $6.40 \pm 2.51$ &
  $4.98 \pm 1.47$ &
  $\mathbf{4.82 \pm 1.40}$ &
  $4.93 \pm 1.38$ \\ \hline
\multirow{2}{*}{Panda - Microwave} &
  Success (\%) &
  $\mathbf{100.00}$ &
  $\mathbf{100.00}$ &
  $\mathbf{100.00}$ &
  $\mathbf{100.00}$ &
  OOM &
  $\mathbf{100.00}$ \\
 &
  Quality (rad) &
  $10.27 \pm 4.33$ &
  $8.63 \pm 2.20$ &
  $6.89 \pm 2.53$ &
  $\mathbf{5.90 \pm 1.69}$ &
  OOM &
  $5.93 \pm 1.74$ \\ \hline
\multirow{2}{*}{Fetch - Conveyor} &
  Success (\%) &
  $\mathbf{100.00}$ &
  $98.10$ &
  $99.90$ &
  $\mathbf{100.00}$ &
  $\mathbf{100.00}$ &
  $\mathbf{100.00}$ \\
 &
  Quality (rad) &
  $11.10 \pm 3.56$ &
  $4.15 \pm 4.16$ &
  $5.43 \pm 4.67$ &
  $4.55 \pm 4.02$ &
  $\mathbf{2.70 \pm 2.49}$ &
  $3.94 \pm 3.45$ \\ \hline
\multirow{2}{*}{Fetch - All Stable} &
  Success (\%) &
  $63.60$ &
  $100.0$ &
  $99.80$ &
  $\mathbf{100.00}$ &
  $\mathbf{100.00}$ &
  $\mathbf{100.00}$ \\
 &
  Quality (rad) &
  $34.99 \pm 9.45$ &
  $21.60 \pm 3.75$ &
  $13.73 \pm 3.00$ &
  $11.88 \pm 2.63$ &
  $\mathbf{11.63 \pm 2.55}$ &
  $11.75 \pm 2.57$ \\ \hline
\multirow{2}{*}{Fetch - Cage} &
  Success (\%) &
  $32.10$ &
  $83.50$ &
  $67.90$ &
  $\mathbf{100.00}$ &
  $\mathbf{100.00}$ &
  $\mathbf{100.00}$ \\
 &
  Quality (rad) &
  $39.51 \pm 9.70$ &
  $31.09 \pm 7.33$ &
  $18.05 \pm 4.72$ &
  $17.14 \pm 4.54$ &
  $\mathbf{16.41 \pm 4.76}$ &
  $17.34 \pm 4.22$ \\ \hline
\multirow{2}{*}{Fetch - Shelf} &
  Success (\%) &
  $68.40$ &
  $99.80$ &
  $95.10$ &
  $\mathbf{100.00}$ &
  $\mathbf{100.00}$ &
  $\mathbf{100.00}$ \\
 &
  Quality (rad) &
  $21.23 \pm 9.67$ &
  $15.11 \pm 6.83$ &
  $9.22 \pm 3.60$ &
  $7.58 \pm 3.05$ &
  $\mathbf{7.05 \pm 2.89}$ &
  $7.54 \pm 3.04$ \\ \hline
\multirow{2}{*}{Fetch - Microwave} &
  Success (\%) &
  $88.90$ &
  $95.00$ &
  $95.10$ &
  $\mathbf{100.00}$ &
  OOM &
  $\mathbf{100.00}$ \\
 &
  Quality (rad) &
  $39.83 \pm 10.19$ &
  $21.46 \pm 4.78$ &
  $18.20 \pm 4.34$ &
  $16.42 \pm 4.21$ &
  OOM &
  $\mathbf{16.15 \pm 4.48}$ \\ \hline
\multirow{2}{*}{UR10 - Real} &
  Success (\%) &
  $98.00$ &
  $88.00$ &
  $99.00$ &
  $\mathbf{100.00}$ &
  $\mathbf{100.00}$ &
  $\mathbf{100.00}$ \\
 &
  Quality (rad) &
  $19.69 \pm 9.26$ &
  $8.42 \pm 4.68$ &
  $7.15 \pm 5.07$ &
  $\mathbf{6.48 \pm 3.96}$ &
  $6.60 \pm 3.40$ &
  $6.80 \pm 4.28$ \\ \hline
\end{tabular}
\vspace{-2ex}  
\end{table*}

\subsection{Experimental Setup}
We evaluate the performance of \method on two manipulators: 
%
Panda (7-\DOF) and Fetch (8-\DOF)
in five different environments adapted from ~\cite{chamzas2022-motion-bench-maker} and shown  in~\autoref{fig:exp_env}

\begin{itemize}
    \item \textbf{Conveyor:} A fixed conveyor belt with a rectangular cuboid object that can arrive randomly.
    Goal variation parameters: $x_i \pm 0.2$ m, $y_i \pm 0.4$ m, $\theta_i \pm \frac{\pi}{2}$ rad.


    \item \textbf{All Stable:} A fixed table with a single asymmetric rectangular cuboid that can be placed stably on any of its six faces.
    Goal variation parameters: $x_i \pm 0.6$ m, $y_i \pm 1.0$ m, $\theta_i \pm \frac{\pi}{2}$ rad, where $i \in \{1,\ldots,6\}$ denotes the placement face.

    \item \textbf{Cage:} The \textit{cage} environment~\cite{chamzas2022-motion-bench-maker}, with one rectangular goal object.
    Goal variation parameters: $x \pm 0.35$ m, $y \pm 0.35$ m, $\psi \pm \frac{\pi}{2}$ rad.

    \item \textbf{Shelf:} The \textit{thin-shelf} environment~\cite{chamzas2022-motion-bench-maker}, with five fixed obstacles and one rectangular goal object. 
    Goal variation parameters: $x \pm 0.31$ m, $y \pm 0.5$ m, $\theta \pm \frac{\pi}{2}$ rad, $z \in \{0.35, 0.60, 0.85, 1.10, 1.35\}$ m.

    \item \textbf{Microwave:} A fixed table with an articulated microwave whose door configuration also varies across queries.
    Goal variation parameters: $x \pm 0.8$ m, $y \pm 0.8$ m, $\psi \pm \frac{\pi}{2}$ rad, $\theta_d \in [0, \frac{\pi}{2}]$ rad, where $\theta_d$ is the microwave door angle.
\end{itemize}
%
%
We perform a real world experiment with a physical UR10 robot shown in \autoref{fig:intro} and described in \autoref{sec:experiment-real}.
%

%
In this experimental setting, the robots are tasked with reaching a desired end-effector pose corresponding to a grasp pose for the object. 
Each query specifies a goal pose sampled from the potential task space.
In our setting, the static environment obstacles (e.g., conveyor, shelf, cage structure) remain fixed, while the query-dependent object to be grasped varies across tasks. 
All experiments are conducted in Mujoco simulation on a machine with Intel i7-14700 and 32\,GB RAM. 
For each robot–environment pair, we first generate all the TCRs, followed by inverse kinematics ($\ik$) to obtain joint-space goal configurations. 
The full solution library is generated by solving and storing each TCR. This full library serves as the uncompressed baseline.
Compressed plan libraries are then generated with the three adaptation methods described in~\autoref{sec:methodology-adapt}. We denote them as \method-LI, \method-DMP and \method-STO.
%
%
We measure path quality using joint-space path length. 

\subsection{Offline Library Building}
\label{sec:experiment-offline}
%

We first evaluate the compression performance of \method with three different adaptation methods. 
\autoref{tab:library} summarizes the compression ratio, defined as the percentage reduction in the number of stored paths relative to the full library, together with the average adaptation time and the resulting path quality.

For both Panda-Microwave and Fetch-Microwave, the Full Library and \method-DMP exceeded the available memory on our 32\,GB machine. Their measurements are therefore marked OOM (Out of Memory).

Among the adaptation methods, \method-STO generally achieves the highest compression, reducing the library size by 71--99\%. 
\method-LI achieves similar compression of 63--99\%.
\method-DMP generally has a lower compression rate, achieving 10--99\% compression across other problems. 
Where evaluated, \method-DMP produces shorter paths than \method-LI in simulation, whereas \method-LI produces the shortest paths among the compressed methods in the UR10 experiment.

\begin{figure*}[t!]
    \centering
    \includegraphics[width=\linewidth]{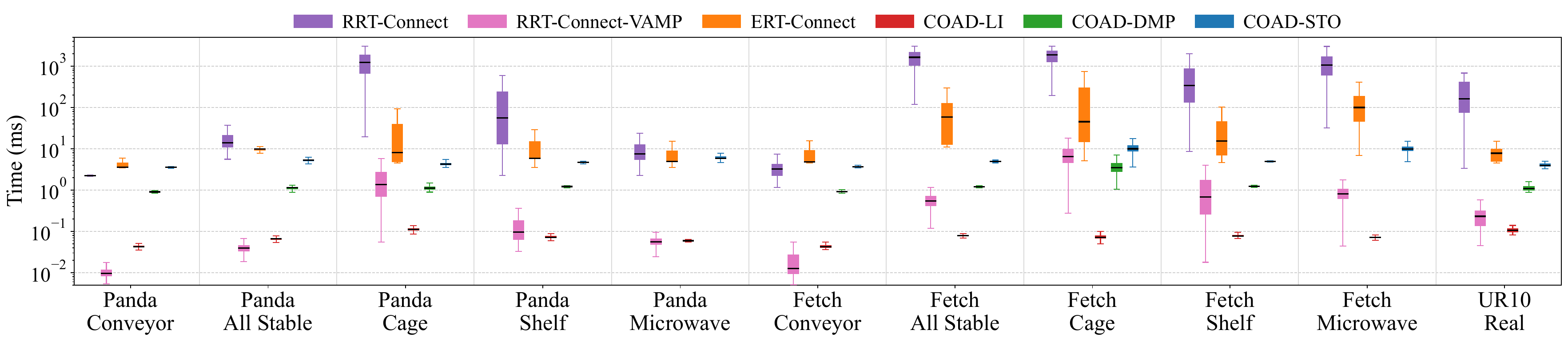}
    \caption{Comparison of Planning Times for all methods. Y-axis is in log-scale milliseconds.} 
    \label{fig:times}
\vspace{-4ex}
\end{figure*}

\subsection{Online Planning}
\label{sec:experiment-online}
We evaluate the online planning performance over 1000 random queries per setting.
The quantitative results for success rate and path quality are summarized in~\autoref{tab:result}.
\method is benchmarked against three baselines: RRT-Connect~\cite{rrt-connect}, RRT-Connect-VAMP~\cite{vamp}, and ERT-Connect~\cite{ERT}.
VAMP accelerates sampling-based planning through vectorized forward kinematics and collision checking. We use its RRT-Connect implementation as a high-performance online planning baseline.
Also, we compare against ERT-Connect \cite{ERT}, a representative experience-based method~\cite{experience_framework, experience_sparse_roadmap, ERT, memory_of_motion} that follows a similar retrieve-and-repair mechanism. Libraries are constructed with $N$ object poses sampled from the task set, where $N$ is the largest compressed library size produced by \method for each problem, capped at $50{,}000$ due to memory limitations.
%
%
%
All methods are executed with a timeout of 3 seconds.

Across all robot–environment pairs, \method achieves a 100\% success rate, while RRT-Connect, VAMP-RRTConnect and ERT-Connect occasionally fail in more constrained environments.
Some VAMP-RRTConnect failures arise from its sphere-based over-approximation of the robot geometry, that make some IK solutions invalid. 
In terms of path quality, all \method adaptations outperform RRT-Connect and RRT-Connect-VAMP while achieving comparable quality to ERT-Connect, with \method-DMP providing the best path quality across most experiments. 

\autoref{fig:times} shows the distribution of planning times. 
%
Among the adaptation methods, \method-LI is the fastest, consistently achieving planning times in the range of 50-90 microseconds. \method-DMP and \method-STO achieve planning times of 1-10 milliseconds.
In terms of planning speed, \method-LI significantly outperforms RRT-Connect and ERT-Connect, achieving up to two orders of magnitude speedup in the open tabletop environments and nearly three orders of magnitude in more constrained environments such as Cage and Shelf. RRT-Connect-VAMP achieves faster planning times in some easier Panda experiments like Panda-Conveyor and Panda-All Stable but struggles to outperform \method when planning complexity increases in the higher DOF Fetch experiments.

In addition to its low mean planning time, \method exhibits substantially smaller planning-time variance than the other baselines across environments. This predictable runtime empirically supports the constant-time online complexity derived in \autoref{sec:methodology-online}.


\subsection{Real-world Validation}
\label{sec:experiment-real}
To validate practical applicability, we evaluate with a UR10 robot in a table scenario as shown in \autoref{fig:intro}. The compressed library is generated in simulation and transferred directly to the physical system.
We assume that the geometry and poses of the fixed obstacles, as well as those of the goal object, are known a priori. During execution, the pose of the goal object is estimated online using an AprilTag-based vision system.

We randomly placed the goal object within the predefined task space for 100 trials.
The observed performance trends, including success rates, query time, and path quality, are consistent with the simulation results and are reported in~\autoref{fig:times} and \autoref{tab:result}.

\section{Conclusion and Future Work}
\label{sec:discussion}

In this paper, we present \method, a framework for goal-varying manipulation planning that combines task-space discretizations via TCRs, offline construction of a compressed library, and efficient online adaptation.
By shifting the majority of the planning and adaptation verification effort to an offline phase, \method enables constant-time query while significantly reducing memory requirements. 

Currently, the root regions are randomly selected.
A promising direction for future work is to design more representative root-selection algorithms to further minimize the number of root paths required.
Another important extension is to relax the semi-static assumption and build a motion library suitable for more general semi-static environments, which includes both varying goals and moving obstacles~\cite{cover}. 
%

\bibliographystyle{bibliography/IEEEtran_ShortURL}
\bibliography{bibliography/references}

\end{document}